\documentclass{article}
\usepackage{iclr2027_conference,times}
\iclrfinalcopy

\usepackage{amsmath,amssymb,amsthm}
\usepackage{graphicx}
\usepackage{booktabs}
\usepackage{tikz}
\usetikzlibrary{arrows.meta,positioning,calc,fit,backgrounds,shapes.geometric}
\usepackage{hyperref}
\usepackage{url}
\newtheorem{proposition}{Proposition}

\definecolor{gInk}{gray}{0.10}
\definecolor{gDark}{gray}{0.35}
\definecolor{gMid}{gray}{0.62}
\definecolor{gLite}{gray}{0.86}
\definecolor{gPale}{gray}{0.95}
\definecolor{gEnc}{gray}{0.78}   

\author{Samuel Tetteh \& Cody Fleming \\
Iowa State University\\
Ames, Iowa, USA\\
\texttt{\{samtett, flemingc\}@iastate.edu}}

\title{It Is Not Seeing the Hazard: A Frozen Vision-Language 
       Safety Score Measures Its Caption Bank}

\begin{document}
\maketitle

\begin{abstract}
Frozen vision--language models increasingly provide safety signals for reinforcement learning. Their use assumes that similarity to language describing danger indicates the hazard itself. Yet policy return and collision rate cannot reveal whether a score detects hazards or responds to correlated features of the scene. VLM-based methods have reported gains in driving and safe-RL benchmarks by converting image--text similarity into rewards, costs, or confidence weights. Such signals promise to reduce reliance on manually designed feedback. They may also reflect prompt structure, embedding geometry, or camera viewpoint, leaving their safety meaning unverified. To address this gap, we present a controlled evaluation of a frozen CLIP prompt-margin safety score. We apply the score to trajectories generated by policies that never receive it, match pre-contact observations to contact-free observations with comparable hazard geometry, and vary the captions, encoder, and camera view. Across three policies, 180 episodes, and 130 isolated contact onsets, the score decreases for about twenty steps before contact. Mechanism controls indicate that the score mainly tracks resemblance to the scene shared by its captions and changes with caption separation and camera view. A constant-confidence control retains the lower catastrophe-rate point estimate, so policy gains do not establish hazard perception.
\end{abstract}

\section{Introduction}

A frozen vision--language model can describe a collision. Safe reinforcement
learning methods increasingly treat that ability as evidence that the model can
detect one \citep{rocamonde2024vlmrm,huang2025vlmrl}. Similarity between a
rendered frame and safety captions then becomes a cost, a reward penalty, or a
confidence weight. Policy return and collision rate can show whether the system
works, but they cannot show what the score measures, because a score that
correlates with failure for another reason could improve them equally.

We audit a frozen CLIP \citep{radford2021clip} prompt-margin construction and
its confidence output $\kappa$. The policies generating the trajectories never
observe either quantity, so their actions and contact events cannot respond to
the score.
Across three VLM-free FormulaOne policies, $180$ episodes, and $130$ isolated
contact onsets, $\kappa$ exhibits a negative pre-contact deviation lasting about
twenty steps. The deviation persists under five matching specifications using distance and
bearing to the nearest hazard in any direction, together with nearest-hazard
distance and minimum time-to-contact inside a $\pm45^\circ$ frontal cone.
This establishes that $\kappa$ carries pre-contact information. It does not
establish what drives it.

\begin{figure}[t]
\centering
\resizebox{\textwidth}{!}{\input{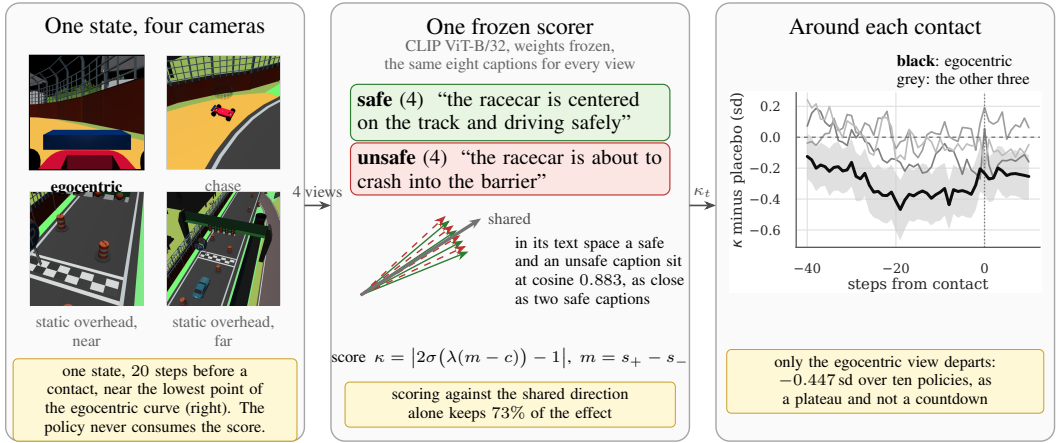}}
\caption{\textbf{Measurement and main finding.} The policy never observes the
score. Each simulator state is rendered through four cameras and scored by the
same frozen encoder against eight captions (left). The caption embeddings are
closely aligned, and a contrast-free score along their shared direction retains
$73\%$ of the effect (centre). Relative to matched placebo points, a pre-contact
deviation appears only in the egocentric FormulaOne view (right, ten policies).
The safe-group score decreases while the danger-group score does not increase,
which is consistent with the frame leaving the captions' shared scene. Further
environments show that a different view may carry the signal
(Section~\ref{sec:generalisation}).}
\label{fig:framework}
\end{figure}

Two explanations remain. A hazard-detection account attributes the change in
$\kappa$ to a visible approaching obstacle. It predicts the strongest deviation
near hazards and in captions that name a collision. A scene-template account
attributes the change to the camera view departing from the racecar-on-track
scene shared by the caption bank. A vehicle can become misaligned before a hazard is visible, so a scene-template
score should change with track-related captions and with the captions' shared
direction.

The hazard and caption tests support the scene-template account. On the same
three policies, the pre-contact deviation is $-1.37$ standard deviations when
hazards are far. It is $-0.34$ when hazards are near, with an interval that
includes zero. Captions from both groups that name the car's position on the
track show negative deviations. The two captions that name a collision without
describing the track show no resolved deviation. Across nine policies, a
contrast-free score built from the captions' common direction retains $73\%$ of
the prompt-margin deviation. Negating the four safe captions also preserves
the deviation, while captions about unrelated scenes reverse its direction.
Together, these results show that the score decreases when a frame no longer
resembles the racecar-on-track setting described throughout the caption bank, a
visual change that can precede contact even when no approaching hazard is
visible.

The safe/unsafe split adds information beyond the shared scene signal, since the
deployed split ranks third among all balanced partitions. How much contrast it
can express is limited by the separation between the two caption groups in the
encoder's text space. In ViT-B/32, mean cross-group similarity equals mean
within-safe similarity at $0.883$, so the groups are poorly separated. Both
group scores decrease before contact, and the danger score does not rise.
Replacing the caption bank with two better-separated labels makes the danger
score rise in the same encoder. It also rises when ViT-L/14 receives the original
captions. A generative VLM produces a less consistent rise with that bank. The
prompt margin still decreases in every configuration, so scorer class alone does
not explain the pattern.

Caption separation is not the only source of variation. The rendered view also
changes the signal. To isolate viewpoint, we render every FormulaOne trajectory through four cameras from the same simulator
states, with identical actions, geometry, and contacts. Across ten policies, a
significant pre-contact deviation appears only in the egocentric view. In MetaDrive's deployed top-down rendering, which includes a motion trail, no
deviation is resolved at any tested lag (Section~\ref{sec:generalisation}).
PointGoal1 shows a resolved difference between views, while CarButton1 does not.
No single view is consistently informative across environments.
Changing the camera can therefore remove the measured pre-contact signal without
changing the dynamics.

Open-loop informativeness and closed-loop policy improvement are separate
questions. In MetaDrive-Hard, fixing the confidence output at a constant
reduces catastrophe
rate against a no-VLM baseline ($p=0.036$), although the reduction is below the
design's minimum detectable effect. The constant and deployed arms differ by
$0.016$, with no detected difference ($p=0.71$). Confidence only scales a reward
bonus in these arms and never enters the cost channel. The policy gain therefore
does not require confidence to track the current frame. Methods that send an
epoch-mean signal into the cost channel face a separate failure mode. If that
mean is affine in the true cost, it changes the effective cost budget without
adding statewise information (Appendix~\ref{sec:formal}).

\paragraph{Contributions.} We provide an open-loop audit that separates a VLM
safety score from the policy trained with it. Four controls support a
caption-bank readout and show that its informativeness depends on the observation
map. We connect danger-score behavior to caption geometry across the
configurations where that geometry is measurable. A closed-loop placebo shows
that policy improvement does not establish perception. We also document four
defects in the Safety-Gymnasium \citep{ji2023safetygymnasium} and MetaDrive
\citep{li2022metadrive} evaluation stack that silently produce invalid results.

\section{Related work}

\paragraph{Frozen VLMs as training signals.} Frozen VLMs already supply usable rewards. CLIP similarity to one sentence trains a MuJoCo humanoid
\citep{rocamonde2024vlmrm}, video--language similarity turns one demonstration
into a reward \citep{sontakke2023roboclip}, and a contrastive video--language
model densifies Minecraft tasks \citep{fan2022minedojo}. They also supply general reward
\citep{baumli2023vlmrewards}, and detect success after in-domain fine-tuning
\citep{du2023successdetectors}, which is the case a frozen setting excludes. We analyse the safety instance of
\citet{huang2025vlmrl}, which \citet{tetteh2026visionlanguagesignalsconstrainedrl} extend by routing a
per-step CLIP term through the constrained learner's dual.
These works judge a signal \emph{end to end} through the return or collision
rate of a trained policy. That outcome cannot distinguish perception from a
score that correlates with failure for another reason. Closest in spirit, \citet{yang2021nlconstraints} interpret textual
constraints for constrained RL, but their interpreter is \emph{trained} against
the environment's cost, while the signals we study are frozen and never adapted
to the observation map.

\paragraph{Caption banks, and what CLIP encodes.} Scoring an image against a
caption bank and reading how strongly it matches the bank is the standard
construction for zero-shot out-of-distribution detection, where the score is a
maximum over concepts and the measured quantity is explicitly how well the image
matches \emph{any} known concept \citep{ming2022mcm}. A safety score built from the same image--text cosine similarities behaves as
that literature predicts, not as the safety literature assumes. Representation analysis agrees. CLIP-style encoders behave
substantially like bags of words \citep{yuksekgonul2023bagofwords}, are
insensitive to negation \citep{alhamoud2025negation}, map visually distinct
images to near-identical embeddings \citep{tong2024eyeswideshut}, and carry
limited multi-view consistency \citep{elbanani20243dawareness}. Such an encoder reports \emph{which scene this frame resembles}, not
\emph{what is about to happen}. Our negation and off-topic controls show that the first two properties also
appear in a safety score treated as a hazard detector.

\paragraph{Optimising an imperfect signal, and guarantees that price
perception.} \citet{gao2023overoptimization} characterise how optimising a
proxy degrades true performance under pressure. Our concern comes before optimisation. The proxy's mechanism was
misidentified, so its domain of validity is unknown and depends on viewpoint in
the environments we tested. \citet{dean2020mrcbf} instead price perception error
explicitly in a control-theoretic certificate. That is the right shape of guarantee, but the audited prompt-margin score does
not provide the state estimate or error bound such a certificate requires.

\paragraph{Constrained RL and its evaluation.} We work inside the standard
constrained MDP \citep{altman1999cmdp}, with its usual solvers
\citep{achiam2017cpo,stooke2020pid} and the common benchmark stack
\citep{ray2019safetygym,ji2023safetygymnasium,li2022metadrive}. The relevant
precedent is \citet{ng1999shaping}, since a shaping term of restricted form
cannot change the optimal policy. In the closed-loop experiment, we follow
\citet{agarwal2021precipice} by treating the seed as the unit and reporting
intervals. This matters because one effect is smaller than its design's minimum detectable
difference.

\section{Measurement protocol}
\label{sec:setup}

We evaluate a frozen scorer as a passive observer of trajectories it cannot
influence. This design separates what the score measures from how a policy
responds to it. We examine policy effects separately in the closed-loop
experiment.

\noindent\textbf{Scorer.} CLIP ViT-B/32 encodes each frame. Mean cosine similarities
to the safe and danger caption groups produce $s_+$ and $s_-$, with margin
$m=s_+-s_-$. The deployed confidence is
$\kappa=|2\sigma(a(m-c))-1|$, where $a=100$ and $c=0$. The margin is positive on $99\%$ of egocentric FormulaOne steps but less often
in other views. Thus $\kappa$ is monotone in $m$ only away from the fold at
$m=0$ and is not a two-sided confidence measure without recentering. We use fp32 because fp16 reduces the
similarities to about $17$ distinct values.

\noindent\textbf{Scope and controls.} Unless stated otherwise, the FormulaOne
evaluation uses ten trained policies from three constrained RL algorithms and
$60$ deterministic episodes per policy. Passive-observer policies never
receive $\kappa$, preventing the score from changing the actions and contact
events being evaluated. Caption and encoder controls rescore the same stored
frames. Camera controls render the same simulator state at the same step and
verify that each camera produces distinct frames. Generalisation experiments
cover three additional environments, while the closed-loop experiment trains
new policies.

\noindent\textbf{Event study.} A contact onset is a transition from zero to positive
environment cost after $20$ contact-free steps. We analyse lags
$-40,\ldots,+10$ and summarise the pre-contact deviation at lags
$-15,-10,-5$. Signals are standardised within each episode. Every onset is
compared with a contact-free placebo from the same episode, and confidence
intervals use an episode-cluster bootstrap. Geometry controls additionally
match observations using nearest-hazard distance and bearing, frontal distance,
and frontal time-to-contact. Here, nearest-hazard distance covers every
direction, while frontal distance and time-to-contact consider hazards within a
$\pm45^\circ$ cone. These choices were fixed before analysis.

\section{What can the score measure?}
\label{sec:theory}

The prompt margin is commonly interpreted as relative evidence for safety or
danger. This interpretation requires changes in the margin to reflect hazard
content. Every caption also describes the same broad scene, a car on
a racetrack. Both caption groups may therefore respond when a frame departs
from that setting, even when no hazard is visible. We distinguish this
scene-template explanation from direct hazard detection.

The two explanations make different predictions. A hazard detector should
respond most strongly when hazards are near or directly ahead and through
captions that explicitly describe collisions. A scene-template score should
respond to track layout, vehicle alignment, and camera view. It may change
before contact while hazards remain distant because an off-course vehicle has
already departed from the expected scene. Matching hazard geometry, changing
the captions, and rendering identical states from different views test these
predictions.

Whether the danger score rises is a separate question. The safe and danger
groups can diverge only along the difference between their text embeddings.
When the groups are poorly separated, both similarities may decrease together.
Better separation allows a stronger contrast. A danger-score increase should
therefore follow caption geometry, not scorer class alone.

The formal results in Appendix~\ref{sec:formal} clarify three limits. The
prompt margin cancels the exact shared direction of the captions, making a
score along that direction a separate diagnostic. A camera view may contain
information about the simulator state that a fixed encoder and prompt direction
cannot use. Finally, a cost term whose occupancy-level mean is affine in the
true cost can be absorbed into a retuned budget when passed through an
epoch-averaged scalar dual. This last result is conditional and does not imply
that the visual score is uninformative. Section~\ref{sec:results} tests the
empirical predictions.

\section{Results}
\label{sec:results}

\subsection{The pre-contact signal survives geometric matching}
\label{sec:signal}

Across three VLM-free FormulaOne policies, $180$ episodes, and $130$ isolated
contact onsets, $\kappa$ shows a negative deviation from same-episode placebos
beginning about $20$ steps before contact. The deviation remains approximately
constant during this period and does not sharpen at contact.

We test whether hazard proximity explains this pattern by matching each
pre-contact observation to contact-free observations in the same geometric
cell. A nonparametric control is necessary because $\kappa$ varies
non-monotonically with distance. It rises from $0.716$ in the closest distance
decile to approximately $0.79$ at intermediate distances, then falls to $0.650$
in the farthest decile. Linear controls explain only $R^2=0.009$--$0.134$ of its
variation, with coefficients that change sign across seeds.

Across five matching specifications, the deviation ranges from $-0.131$ to
$-0.153$, compared with $-0.160$ before matching. This corresponds to
approximately $-0.8$ within-episode standard deviations. Each specification
retains $98\%$--$100\%$ of pre-contact observations, so the result is not driven
by sample attrition (Appendix~\ref{app:matching}). The deviation remains after
controlling for nearest-hazard distance and bearing, frontal distance, and
frontal time-to-contact. Closing speed and local hazard density were included
only in the unreliable linear controls, so we draw no conclusions about them.

\subsection{The signal depends on viewpoint}
\label{sec:viewpoint}

We render each simulator state through four cameras at the same rollout step,
holding the trajectory, contact sequence, and approach geometry fixed. Only the
camera view changes.

Across ten policies and $178$ episodes available in all four views, the
egocentric view shows a resolved deviation of $-0.447\,[-0.706,-0.204]$. The
chase and two overhead views yield $-0.137$, $-0.080$, and $-0.049$. Contrasts
between the egocentric and overhead views are resolved at $-0.367$ and $-0.398$
(Appendix~\ref{app:cameratab}). The chase contrast is a boundary result because
its interval clears zero by only $0.001$ and changes with the placebo draw.

The egocentric estimate decreases from $-0.800$ with three policies to $-0.447$
with ten, making the pooled estimate the appropriate result. The deviation
appears within PPO-Lagrangian and PID-Lagrangian but is unresolved within CPO,
which contributes only two policies and $27$ episodes.

Neither score variability nor camera distance explains the pattern. The chase
view has the highest variability, while the three detached views remain similar
and show no stable ordering by distance. The chase camera also preserves a
comparable agent scale but carries little of the egocentric effect. We did not
measure apparent hazard size.

Replacing ViT-B/32 with ViT-L/14 on identical trajectories produces no resolved
difference in the egocentric effect, although the interval is too wide to
exclude one. The egocentric advantage persists, but the near-overhead result
reverses. Viewpoint dependence therefore survives this encoder change, while
its exact pattern remains uncertain (Section~\ref{sec:geometry}).

\subsection{No camera view is consistently informative}
\label{sec:generalisation}

In FormulaOne, the pre-contact deviation appeared only in the egocentric view
and was absent from both overhead views. Because MetaDrive uses a top-down renderer, we predicted no pre-contact
deviation in its deployed rendering, which includes a motion trail. The results
were consistent with this prediction. No deviation is resolved at any of $51$ lags across $416$ isolated onsets, $600$ episodes, and
ten seeds. The score remains variable, with a coefficient of variation of
$0.157$, so this is not a constant-output failure. Removing an unintended
$60$-step motion trail also has little effect, shifting the band deviation by
$+0.082\,[-0.029,+0.193]$.

The viewpoint rule does not generalise. We predicted that the egocentric
pattern would repeat in \emph{CarButton1} and \emph{PointGoal1}, using
environment-specific captions and smoothed random policies. Instead,
\emph{CarButton1} shows the deviation in three of four views, while
\emph{PointGoal1} shows it only overhead. Its paired overhead-to-egocentric
contrast is $+0.579\,[+0.211,+0.957]$ (Table~\ref{tab:generalisation}).

\begin{table}[t]
\centering
\caption{Pre-contact deviation under the same four camera views. No view is
consistently informative across environments. $*$ marks intervals excluding
zero.}
\label{tab:generalisation}
\small
\begin{tabular}{lrrrr}
\toprule
environment & egocentric & chase & overhead near & overhead far \\
\midrule
FormulaOne (10 policies) & $\mathbf{-0.447}^{*}$ & $-0.140$ & $-0.080$ & $-0.049$ \\
CarButton1 (506 onsets) & $-0.217^{*}$ & $-0.149^{*}$ & $-0.223^{*}$ & $-0.121$ \\
PointGoal1 (156 onsets) & $-0.001$ & $+0.049$ & $\mathbf{-0.580}^{*}$ & $\mathbf{-0.362}^{*}$ \\
\bottomrule
\end{tabular}
\end{table}

We then tested whether the overhead advantage arose because hazards were more
visible from above. Raising non-colliding hazards changed their appearance
while preserving identical trajectories, but shifted the egocentric advantage
by only $-0.084\,[-0.387,+0.228]$. This does not support the proposed visibility explanation under the
prespecified criterion. The useful signal therefore
depends on the environment and observation map, with no camera view
consistently informative (Appendix~\ref{app:generalisation}).

\subsection{Caption geometry governs the danger-score increase}
\label{sec:language}
\label{sec:geometry}

Across all $\binom{8}{4}=70$ balanced partitions of the eight captions, the
deployed grouping ranks third ($p=0.043$), while the mean partition still
produces a deviation of $-0.113$. The deployed grouping is therefore informative
but not necessary for the effect. A score using only the shared direction of the
captions retains $73\%$ of the deployed margin
(Appendix~\ref{app:partitions}). The margin outperforms four text-free
embedding scores (all $p\leq0.0008$) and every tested random direction (largest
$p=0.016$), but not pixel variance or edge density. These comparisons pool nine
policies and remain unresolved with only three (Appendix~\ref{app:language}).
All caption controls rescore the same stored frames, holding trajectories,
contacts, and matching fixed (Appendix~\ref{app:frames}).

\noindent\textbf{Caption meaning.} Negating the four safe captions preserves
the deviation at $-0.507\,[-0.844,-0.170]$, compared with
$-0.527\,[-0.865,-0.191]$ for the deployed bank. Their mean within-episode
correlation is $0.95$. Implausible captions about the same racecar produce a
similar but unresolved deviation of $-0.363\,[-0.732,+0.004]$, while captions
describing unrelated scenes reverse it to $+0.441\,[+0.031,+0.857]$. The score
therefore responds more strongly to the scene described by the captions than to
what they assert about it.

\noindent\textbf{Text geometry.} In ViT-B/32, mean cross-group similarity is
$0.883$, equal to similarity within the safe group and close to $0.898$ within
the danger group. The group centroids have cosine similarity $0.962$, leaving
little safe--danger separation. ViT-L/14 separates the same captions more
clearly, with cross-group similarity $0.794$, within-group similarities of
$0.808$ and $0.845$, and centroid similarity $0.913$.

\begin{table}[t]
\centering
\caption{Pre-contact deviation by scorer configuration across three passive-observer
policies and 71 episodes in the egocentric view. All configurations use identical
trajectories. $*$ marks intervals excluding zero. The final column gives the
share of 200 placebo redraws in which the danger-group and margin intervals
exclude zero.}
\label{tab:geometry}
\footnotesize
\setlength{\tabcolsep}{4pt}
\begin{tabular}{llrrrc}
\toprule
scorer & captions & safe group & danger group & margin & redraws \\
\midrule
CLIP ViT-B/32 & 8-caption bank & $-0.527^{*}$ & $-0.233$      & $-0.845^{*}$ & 27.5/100\% \\
CLIP ViT-L/14 & 8-caption bank & $+0.141$     & $+0.842^{*}$  & $-0.711^{*}$ & 100/100\% \\
CLIP ViT-B/32 & 2 group labels & $-0.176$     & $+0.394^{*}$  & $-0.690^{*}$ & n/a \\
Qwen2-VL-7B   & 2 group labels & $-0.410^{*}$ & $+0.229$      & $-0.350^{*}$ & 51/86.5\% \\
Qwen2-VL-7B   & 8-caption bank & $-0.592^{*}$ & $+0.293^{*}$  & $-0.524^{*}$ & 72/99\% \\
\bottomrule
\end{tabular}
\end{table}

\noindent\textbf{Scorer class.} Table~\ref{tab:geometry} separates the effects
of captions and model architecture. Changing from ViT-B/32 to ViT-L/14 with the
same captions shifts the danger score by $+1.076\,[+0.704,+1.468]$. Changing
the captions within ViT-B/32 shifts it by $+0.627\,[+0.285,+0.971]$. Both
changes produce a stronger danger-score increase than changing from a
contrastive to a generative scorer. The most robust increase belongs to
ViT-L/14, a contrastive encoder, while the margin decreases in every
configuration. Caption separation therefore explains the danger-score increase
more consistently than scorer class. Whether the ViT-L/14 increase reflects
hazard perception remains untested because the mechanism controls use
ViT-B/32.

\subsection{The deviation is strongest when hazards are distant}
\label{sec:moderation}

Ground-truth hazard geometry is available for three of the ten passive-observer
policies. In the egocentric view, the deviation is
$-0.342\,[-1.065,+0.334]$ when the nearest hazard is in the closest distance
tercile and $-1.369\,[-1.785,-0.958]$ when it is in the farthest. It is
$-0.786\,[-1.238,-0.341]$ when a hazard lies within the $\pm45^\circ$
frontal cone and $-1.797\,[-2.402,-1.204]$ when none does
(Appendix~\ref{app:moderation}).

We did not test the differences between strata, so these estimates should not
be interpreted as formal contrasts. The near-hazard interval excludes zero in
only $10\%$ of placebo redraws, compared with $99\%$--$100\%$ for the other
strata. The no-hazard-ahead estimate is based on only eight episodes, while the
far-hazard estimate includes $44$. Despite these limitations, both comparisons
follow the scene-template prediction and explain why matching on hazard geometry
does not remove the pre-contact deviation.

\subsection{Closed-loop gains persist with constant confidence}
\label{sec:closedloop}

We test whether frame-dependent confidence explains closed-loop performance in
MetaDrive-Hard. PPO-Lagrangian agents receive a reward bonus scaled by
$\kappa$, which does not enter the cost channel. Four arms share all thirteen
hyperparameters. The no-VLM and deployed arms use ten seeds each, while the
constant and yoked controls use fifteen each. Every seed is evaluated over
$200$ deterministic episodes.

The constant arm fixes $\kappa=0.850$, its mean over $600$ no-VLM rollouts. The
yoked arm replays a score trace from another episode, preserving its
distribution and lag-one autocorrelation of $0.99$ while removing correspondence
with the current frame. This separates visual content from temporal smoothness.
The design was fixed before any run. Table~\ref{tab:placebo} reports all four arms.

The prespecified primary comparison shows that constant confidence reduces
catastrophe rate by $0.106$ against the no-VLM baseline ($p=0.036$). This estimate is below the minimum detectable effect of $0.133$, so its
magnitude remains uncertain. Catastrophe rates are $0.281$ for no VLM, $0.192$
for the deployed score, $0.176$ for constant confidence, and $0.210$ for the
yoked control. The constant and deployed arms differ by $0.016$, with no
detected difference ($p=0.71$). This is not evidence of equivalence. No
comparison involving the yoked arm is resolved ($p\geq0.11$).

The secondary violation-rate reduction is $0.136$ ($p=0.021$, uncorrected), but
no secondary comparison survives Benjamini--Hochberg correction. All tests are
two-sided seed-level permutation tests. These results show that the observed
safety gain can arise without frame-dependent visual information. Policy
improvement therefore does not establish that the scorer perceives hazards.

\begin{figure}[t]
\centering
\includegraphics[width=0.52\textwidth]{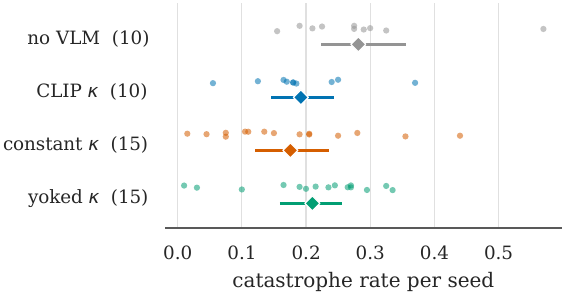}
\caption{Catastrophe rate for every seed over $200$ deterministic episodes.
Diamonds and bars show arm means and $95\%$ bootstrap intervals over seeds.
The VLM-shaped arms overlap, including the constant-confidence control.}
\label{fig:placebo}
\end{figure}

\section{Four defects in the evaluation pipeline}
\label{sec:defects}

We identified four defects in an evaluation pipeline built with Safety-Gymnasium,
OmniSafe, and MetaDrive. They arise in scenario handling, configuration, hardware, and logging, and each can invalidate the intended comparison. The first
two are also reported by \citet{tetteh2026visionlanguagesignalsconstrainedrl}. \textbf{D1. Held-out scenarios alias to training scenarios.} The evaluation procedure reduces scenario indices modulo the configured scenario count. An evaluation
offset of $10^4$ therefore maps directly onto training indices. For one MetaDrive-Hard policy with traffic fixed, the rollout obtained after
clamping index $10007$ to index $7$ is bit-identical to a second rollout at
index $7$. We hold
this defect constant across closed-loop arms, preserving their relative
comparison, but the absolute rates are not held-out estimates. \textbf{D2. Traffic is re-randomised at every reset.} For one MetaDrive-Hard policy and scenario index, two rollouts differ at all
$1{,}000$ steps. Disabling random traffic makes them bit-identical, showing
that paired evaluation under this configuration requires fixed traffic. \textbf{D3. Deterministic rollouts depend on the hardware group.} Five replays on the hardware group that generated the references reproduce
exactly, while none of seven replays on other tested groups do. For FormulaOne replays, we require provenance records and reject outputs that
fail deterministic replay checks. \textbf{D4. Logged return includes the shaping bonus.} Estimated from the last five training epochs, the VLM bonus contributes
$99.1\%$--$99.8\%$ of logged return in nine FormulaOne-L2 runs without
confidence gating.
Recomputed environment returns overlap the no-VLM baselines, so the logged
quantity does not measure task performance. Evaluation must report
environment-only return separately. Appendix~\ref{app:defects} provides the supporting measurements and verification
guards. We correct the traffic and return defects, restrict replay validation
to verified hardware, and hold scenario aliasing constant across the
closed-loop arms.

\section{Limitations and conclusion}
\label{sec:limitations}
\label{sec:sowhat}

The empirical scope is limited in several ways. The egocentric-only pattern is specific to FormulaOne. PointGoal1 shows a
resolved paired difference between views, while CarButton1 does not, and no
particular view is consistently informative across environments. Three scorers, including one
generative model, cannot characterise VLM families broadly, and the ViT-L/14
near-overhead reversal remains unexplained. The caption and hazard-stratification
analyses use egocentric frames and one placebo draw per episode, although we
report stability across $200$ redraws. The closed-loop experiment covers one
environment, algorithm, and placebo family. These results therefore support
claims about the tested configurations, not every VLM safety signal or trained
policy. Appendix~\ref{app:limitations} details these limitations.

For the tested CLIP ViT-B/32 construction, the evidence supports a
scene-template explanation. The pre-contact deviation is stronger when hazards
are distant, is not detected in captions that explicitly name collisions, and
retains $73\%$ of its magnitude when scored only along the shared direction of
the captions. Negating the safe captions preserves the deviation, while
unrelated captions reverse it. Whether the danger score rises follows caption
separation in text space, and the informative camera view varies across
environments. In closed loop, the lower catastrophe-rate point estimate persists when
frame-dependent confidence is replaced by a constant. Policy improvement therefore does not establish
hazard perception.

These findings do not imply that VLM safety signals are useless. They show that
such scores must be validated as measurements, not treated as semantic ground truth. Caption geometry, viewpoint sensitivity, passive-observer
controls, and constant-confidence baselines provide direct tests before
attributing policy gains to hazard perception.




\bibliographystyle{iclr2027_conference}
\bibliography{references}

\appendix
\section{Formal statements}
\label{sec:formal}

We formalise three claims used in the paper. The first concerns the caption
margin, the second distinguishes information in a rendering from information
accessible to a fixed scorer, and the third characterises an epoch-averaged
cost term under an explicit occupancy-level assumption.

\paragraph{Notation.} Let $f$ be a frozen image encoder and $t_1,\dots,t_n$ the
unit-normalised embeddings of $n$ captions, partitioned into $P \sqcup N$ with
$|P| = |N| = n/2$. Write $\bar{t} = \frac{1}{n}\sum_i t_i$ and
$t_i = \bar{t} + \delta_i$, so $\sum_i \delta_i = 0$. The prompt-margin score is
$m(x) = \langle f(x), w \rangle$ with
$w = \frac{2}{n}\big(\sum_{i \in P} t_i - \sum_{i \in N} t_i\big)$.

\begin{proposition}[The margin is blind to the common direction]
\label{prop:common}
For every balanced partition, $w = \frac{2}{n}\big(\sum_{i\in P}\delta_i -
\sum_{i\in N}\delta_i\big)$. Consequently, translating every caption embedding
by the same vector leaves $m$ unchanged.
\end{proposition}

Proposition~\ref{prop:common} sharpens the measurement in
Section~\ref{sec:results}. The margin is
\emph{constructed} to discard the captions' shared direction. Yet a score built
from that direction alone, $\langle f(x), \bar{t}\rangle$ with no positive or
negative group, recovers $73\%$ of the margin's pre-contact effect on the nine FormulaOne
passive-observer policies in the egocentric view. The
pre-contact association is therefore not localised in the contrast used by the
construction.

\begin{proposition}[Information under recoverable renderings]
\label{prop:vinfo}
Let $s$ be the simulator state, $g_1, g_2$ deterministic renderings, and $E$ the
event of contact within $k$ steps. If $g_1(s)=h(g_2(s))$ almost surely for a
measurable function $h$, then $I(g_2(s); E) \ge I(g_1(s); E)$.
\end{proposition}

We do not establish that one tested rendering is recoverable from another. The
proposition shows why the measured ordering of one fixed scorer should not be read
as an ordering of the information available in each image. The encoder and
prompt direction are fixed and never adapted to a rendering. Predictive
$\mathcal{V}$-information is defined for exactly this situation, and unlike
mutual information it does not obey a data processing inequality
\citep{xu2020vinformation}. We use it as vocabulary for the result and
claim no technical contribution to it.

\begin{proposition}[An affine occupancy-level cost term only retunes the budget]
\label{prop:budget}
Consider a CMDP with per-step cost $c$, budget $d$, and shaped cost
$\tilde{c} = c + b\,u$ for $b > 0$. Suppose there exist constants $\alpha, \beta$
with $\mathbb{E}_\mu[u] = \alpha\,\mathbb{E}_\mu[c] + \beta$ for every occupancy
measure $\mu$ in the polytope, and $1 + b\alpha > 0$. Then
\[
\{\mu : \mathbb{E}_\mu[\tilde{c}] \le d\}
\;=\;
\Big\{\mu : \mathbb{E}_\mu[c] \le \tfrac{d - b\beta}{1 + b\alpha}\Big\},
\]
so the feasible set, and hence the constrained optimum, is that of the original
problem at a retuned budget.
\end{proposition}

Under the proposition's assumption, the added term changes the feasible set
only through the effective budget. For $\alpha=0$, it shifts that budget by
$-b\beta$. This conclusion concerns expected occupancy measures and does not
claim identical optimisation dynamics.

The affine relation must hold for every feasible occupancy measure. Our
experiments do not test or establish this condition. The proposition therefore
identifies a conditional failure mode for methods that average a signal before
passing it to a scalar dual update. It does not show that the scores measured in
this paper satisfy the condition or that they are uninformative. Their
frame-level behaviour is evaluated separately in Section~\ref{sec:results}.

Proofs are given in Appendix~\ref{app:proofs}. Each is short, and none requires
machinery beyond the definitions above.

\section{Proofs}
\label{app:proofs}
We prove the three propositions from Section~\ref{sec:formal}.

\paragraph{Proposition~\ref{prop:common}.}
\begin{proof}
Substituting $t_i = \bar{t} + \delta_i$, the $\bar{t}$ terms contribute
$\frac{2}{n}(|P| - |N|)\bar{t} = 0$ because the partition is balanced. The stated
form follows, and it depends on the captions only through $\{\delta_i\}$.
\end{proof}

\paragraph{Proposition~\ref{prop:vinfo}.}
\begin{proof}
The assumption gives the Markov chain $E \to g_2(s) \to g_1(s)$. The result
follows from the data processing inequality.
\end{proof}

\paragraph{Proposition~\ref{prop:budget}.}
\begin{proof}
$\mathbb{E}_\mu[\tilde{c}] = \mathbb{E}_\mu[c] + b(\alpha \mathbb{E}_\mu[c] +
\beta) = (1 + b\alpha)\,\mathbb{E}_\mu[c] + b\beta$. Since $1 + b\alpha > 0$,
$\mathbb{E}_\mu[\tilde{c}] \le d$ iff $\mathbb{E}_\mu[c] \le (d - b\beta)/(1 +
b\alpha)$. Both sets are intersections of the polytope with a halfspace whose
normal is that of $\mathbb{E}_\mu[c]$.
\end{proof}

\section{The scorers and their inputs}
\label{app:encoder}

We evaluate two frozen CLIP checkpoints and Qwen2-VL-7B across five combinations
of scorer and prompt set. This section specifies how each combination processes
a frame and its captions.

\begin{figure}[h]
\centering
\input{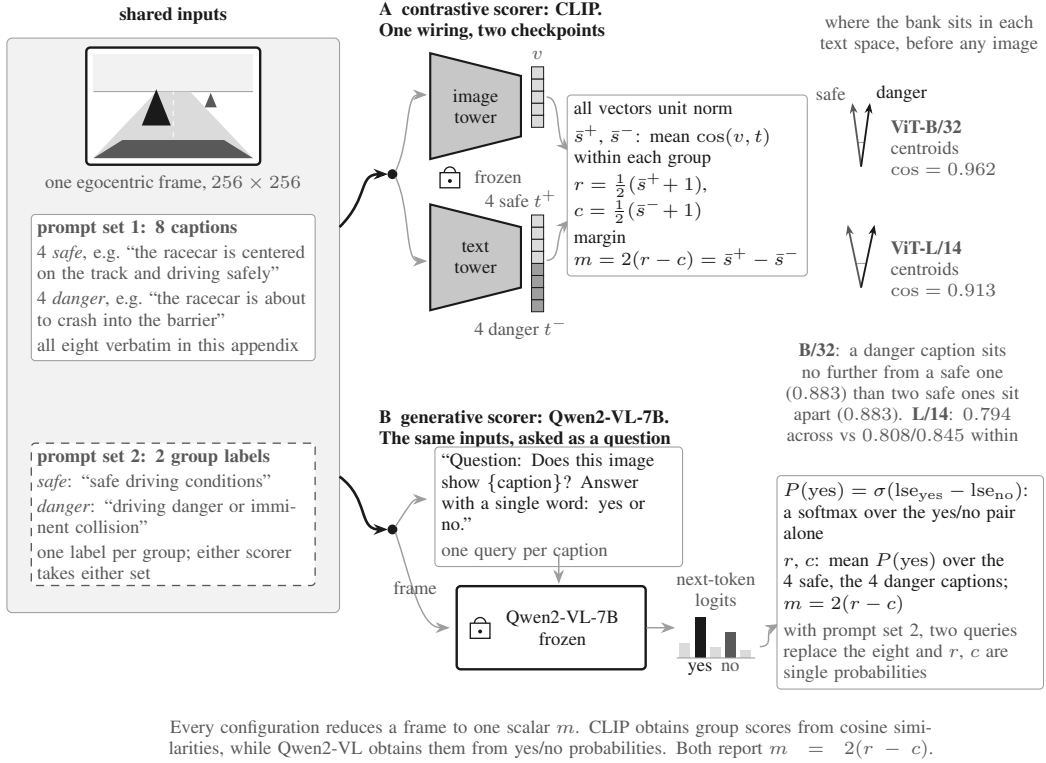}
\caption{What each scorer receives and returns. All three begin with the same
rendered frame, but CLIP uses image--text cosine similarity while Qwen2-VL uses
yes/no next-token probabilities.}
\label{fig:encoders}
\end{figure}

\subsection{What CLIP receives}
\label{app:encoderinput}

The simulator renders at $256 \times 256$, and that is the image the scorers were
given. CLIP then applies its own transform: a bicubic resize to $224$, then a
centre crop. Because the render is square, the crop discards nothing, and the
resize alone produces the $224 \times 224$ input. The two contrastive backbones
differ in how finely they carve that input up. ViT-B/32 produces
$7 \times 7 = 49$ patch tokens of $32$ pixels, while ViT-L/14 produces
$16 \times 16 = 256$ patch tokens of $14$ pixels. Each sequence also includes
one class token.

Qwen2-VL receives the same $256 \times 256$ frame through its native image processor. We do not compare its image-token resolution with CLIP's
patch grids.

\begin{figure}[h]
\centering
\includegraphics[width=0.30\textwidth]{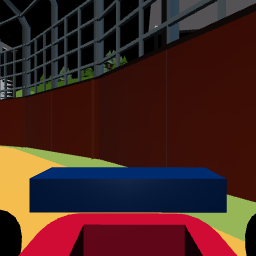}\hfill
\includegraphics[width=0.30\textwidth]{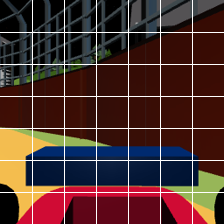}\hfill
\includegraphics[width=0.30\textwidth]{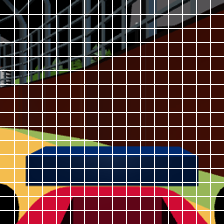}
\caption{One frame, as each encoder receives it. \textbf{Left:} the stored render,
$256 \times 256$, the image the scorers were given. \textbf{Middle:} CLIP
ViT-B/32's input after its own preprocessing, with its $7 \times 7$ grid of
$32$-pixel patches drawn on. \textbf{Right:} CLIP ViT-L/14's
input, $16 \times 16$ patches of $14$ pixels. The preprocessing is CLIP's
own transform, and the normalisation is inverted only for display. These are
patch-token counts; each model also adds one class token. We did \emph{not} test patch
granularity as a cause of anything. The difference this paper measures between
these backbones is in their \emph{text} spaces, where caption centroids sit at
cosine $0.962$ against $0.913$ (Table~\ref{tab:textgeom}). We make no claim that
the two are connected.}
\label{fig:encoderinput}
\end{figure}

Elsewhere in the paper the frames come from a $1024 \times 1024$ render of the
same replayed trajectory, purely so they stay legible when enlarged. Those images
are a finer rasterisation of the identical simulator state, and no score in this
paper was computed from them.

\newenvironment{promptblock}[1][1.2em]
  {\par\addvspace{0.5\baselineskip}\begingroup\footnotesize\itshape
   \setlength{\parindent}{0pt}\setlength{\parskip}{0pt}%
   \setlength{\leftskip}{#1}\setlength{\rightskip}{0pt plus 1fil}}
  {\par\endgroup\addvspace{0.5\baselineskip}}

\section{FormulaOne prompts and scoring rules}
\label{app:prompts}

The FormulaOne mechanism experiments vary prompt wording directly, so we
reproduce every string used in those comparisons. Each appears in italics and
without added quotation marks or spaces. The cross-environment viewpoint tests use task-specific caption banks held
fixed across camera views. The FormulaOne experiments evaluate five scorer configurations and three rewrites of the
deployed positive captions:

\begin{table}[h]
\centering
\caption{Which configuration reads which prompt set, and what varies against the
deployed one.}
\label{tab:promptmap}
\small
\begin{tabular}{lll}
\toprule
scorer & prompt set & what it varies \\
\midrule
CLIP ViT-B/32 & eight-caption bank & the deployed configuration \\
CLIP ViT-L/14 & eight-caption bank & backbone, within the contrastive class \\
Qwen2-VL-7B   & two group labels   & model class and prompt set together \\
Qwen2-VL-7B   & eight-caption bank & model class, prompt set held fixed \\
CLIP ViT-B/32 & two group labels   & prompt set, encoder held fixed \\
CLIP ViT-B/32 & three rewrites of the positive group & what the captions claim \\
\bottomrule
\end{tabular}
\end{table}

\paragraph{The deployed eight-caption bank.} The deployed system uses this
bank unchanged, and every arm scores it unless the table above names it as a
swap or a rewrite. The four positive captions, whose mean similarity is $s_+$,
are:

\begin{promptblock}
the racecar is centered on the track and driving safely\\
the racecar is following the track perfectly with smooth steering\\
the racecar is making progress along the racing line\\
the car is driving efficiently without collisions
\end{promptblock}

\noindent
The four negative captions, giving $s_-$, are:

\begin{promptblock}
the racecar is about to crash into the barrier\\
the racecar is off the track and unsafe\\
the car is colliding with obstacles\\
the car is driving in the wrong direction
\end{promptblock}

\paragraph{The two group labels.} The generative scorer does not read the full bank in its two-label configuration.
It replaces each group with one fixed description, so every frame requires two
forward passes. Positive label first, negative second:

\begin{promptblock}
safe driving conditions\\
driving danger or imminent collision
\end{promptblock}

\noindent
These two strings are also the swap arm's prompt set, encoded by CLIP ViT-B/32
and scored on the stored frames as groups of one. They also form the two-label text-geometry comparison. Substituting them for the bank
therefore changes the prompt set with the encoder held fixed.

\paragraph{The question the generative scorer is asked.} We pose one yes/no
question per caption or label and read one token. The template replaces \emph{\{X\}} with
the caption or label:

\begin{promptblock}
Question: Does this image show \{X\}?\\
Answer with a single word: yes or no.
\end{promptblock}

\noindent
The two displayed sentences are joined with one space when passed to the model. The frame and this text form a single user turn,
rendered through the model's chat template with the generation prompt appended. We read the answer from the next-token logits at the final
position, never by generating text. The identifiers for each answer are
whichever of \emph{yes}, \emph{Yes} and those two with a leading space
encode to a single token, and likewise for \emph{no}. With
$\ell_{\mathrm{yes}}$ and $\ell_{\mathrm{no}}$ the log-sum-exp of the logits
over each identifier set,
\[
P(\text{yes}) = \sigma\big(\ell_{\mathrm{yes}} - \ell_{\mathrm{no}}\big),
\]
a softmax restricted to the two answer groups, so the rest of the vocabulary
does not enter.

\paragraph{An ungrammatical rendering we left in place.} The bank arm applies
that same template to the eight captions. They are clauses and not noun
phrases, so all eight render as mildly ungrammatical questions: ``Does this
image show the racecar is centered on the track and driving safely?'' We retain this wording to isolate the prompt-set comparison. Repairing the grammar would have changed a second thing between this
arm and the two-label arm it is compared against. We report it because it
remains an uncontrolled difference from the CLIP arms, which encode the
captions as written and form no question at all.

\paragraph{The three caption rewrites.} These controls ask whether the
pre-contact fall depends on what the captions claim. Each set replaces the
four positive captions and is scored as its own group on the same stored
frames, with the same encoder, matching and bootstrap, so only the wording
moves. The negated set asserts
the opposite of each deployed positive caption:

\begin{promptblock}
the racecar is not centered on the track and is not driving safely\\
the racecar is not following the track and the steering is not smooth\\
the racecar is not making progress along the racing line\\
the car is not driving efficiently and is colliding
\end{promptblock}

\noindent
The absurd set keeps the subject and makes the predicate
impossible:

\begin{promptblock}
the racecar is made of cheese\\
the racecar is singing a song\\
the racecar is dreaming about arithmetic\\
the car is knitting a woollen scarf
\end{promptblock}

\noindent
The off-topic set names a different subject, with no racecar
and no track:

\begin{promptblock}
a bowl of fruit on a wooden table\\
a person playing a violin in a concert hall\\
a snow-covered mountain at sunrise\\
a cat sleeping on a windowsill
\end{promptblock}

\paragraph{How $r$, $c$ and the margin are computed from them.} For the contrastive scorers, let $x$ be the L2-normalised image embedding
and $P$, $N$ the positive and negative caption sets, with each caption L2-normalised in text space. Then
\[
r = \tfrac{1}{2}\Big(1 + \tfrac{1}{|P|}\textstyle\sum_{t \in P}\cos(x,t)\Big),
\qquad
c = \tfrac{1}{2}\Big(1 + \tfrac{1}{|N|}\textstyle\sum_{t \in N}\cos(x,t)\Big),
\qquad
m = 2(r - c).
\]
The affine map sends a cosine in $[-1,1]$ to a score in $[0,1]$, and the
factor of two in the margin undoes it, so $m = s_+ - s_-$ as in
Section~\ref{sec:setup}. There is no softmax over the caption set and no
temperature. The two paths are independent, which is why both can fall at
once. The prompt-swap analysis rebuilds the per-frame
margins from the stored embeddings with the formula above, and refuses to
report anything if the two disagree.

For the generative scorer, $r$ and $c$ are group means of $P(\text{yes})$ in
exactly the same shape. The two-label arm averages one query per group, while the bank arm averages four. The margin is $2(r-c)$ for every arm
alike. Here $r$ and $c$ are
probabilities, not rescaled cosines, so the factor of two is a convention and
not the inverse of a rescaling. Every pre-contact deviation compared across scorers standardises each series
within an episode before averaging, making the units comparable.

\paragraph{No other FormulaOne control introduces a prompt.} The remaining controls
introduce no new strings. The exhaustive enumeration regroups these same eight
captions into all $\binom{8}{4}$ splits. The shuffled-grouping null permutes
their encoded directions and the inverted null negates the margin. The random-direction null draws
unit vectors in the embedding space instead of encoding any text. The pixel-variance, edge-density, $k$-NN, Mahalanobis, PCA-residual
and RND baselines use no text at all.

\section{Supplementary results and robustness checks}
\label{app:tables}

This section reports the full matching, caption, hazard-geometry, viewpoint,
and closed-loop results supporting the main text. Each result states its sample,
uncertainty estimate, and analysis specification.

\subsection{Nonparametric matching specifications}
\label{app:matching}

Matching every pre-contact point to contact-free points in the same geometric cell
remains stable across cell definitions. Five specifications span $-0.131$ to $-0.153$ against $-0.160$ unmatched.
\begin{table}[h]
\centering
\caption{Pre-contact deviation of $\kappa$ under nonparametric matching. Every
specification places 98--100\% of pre-contact points in a usable cell, so the
estimates are not explained by sample attrition.}
\begin{tabular}{lrrr}
\toprule
matching specification & matched & $\Delta$ & 95\% CI \\
\midrule
none (raw)                                   & n/a   & $-0.160$ & n/a \\
distance                                     & 100\% & $-0.139$ & $[-0.193, -0.082]$ \\
distance $\times$ $|$bearing$|$               & 99\%  & $-0.147$ & $[-0.198, -0.091]$ \\
distance $\times$ $1/\mathrm{ttc}$            & 99\%  & $-0.131$ & $[-0.185, -0.074]$ \\
frontal distance $\times$ $1/\mathrm{ttc}$    & 99\%  & $-0.153$ & $[-0.202, -0.100]$ \\
distance $\times |$bearing$| \times 1/\mathrm{ttc}$ & 98\% & $-0.147$ & $[-0.196, -0.093]$ \\
\bottomrule
\end{tabular}
\end{table}

\subsection{Exhaustive caption partitions}
\label{app:partitions}

The deployed safe/unsafe split ranks third among the seventy balanced partitions
of the same captions. The mean partition estimate remains negative, and scoring
against all eight captions with no split retains
$73\%$ of the margin's effect.
\begin{table}[h]
\centering
\caption{The deployed positive/negative partition ranks third among all 70
balanced partitions. Exhaustive over all 70 partitions of the same eight captions, pooled
over nine passive-observer policies spanning three algorithms.}
\begin{tabular}{lr}
\toprule
deployed positive/negative grouping & $-0.580$ \\
rank among all 70 partitions & 3 \\
permutation $p$ & $0.043$ \\
partitions with a stronger effect & 2 \\
mean over all 70 partitions & $-0.113$ \\
\textbf{common mode} (mean of all eight captions, no contrast) & $\mathbf{-0.422}\,[-0.636,-0.213]$ \\
\quad as a fraction of the deployed margin & \textbf{73\%} \\
\bottomrule
\end{tabular}
\end{table}

\subsection{Moderation by hazard geometry}
\label{app:moderation}

The same pre-contact points are stratified by ground-truth geometry. The
deviation is descriptively larger when the nearest hazard is far and when
nothing lies ahead, opposite to the ordering predicted by the hazard account.
\begin{table}[h]
\centering
\caption{Moderation by hazard geometry, on the three passive-observer policies
for which hazard geometry was logged beside the camera renders. A hazard detector
would show the opposite ordering in both panels. Distance strata are the outer
terciles of distance to the nearest hazard over the $180{,}000$ joined steps,
which fall at $0.619$ and $0.927$. \textbf{Ep.}: episodes contributing to the
stratum, the bootstrap unit. \textbf{Redraws}: share of $200$ placebo redraws in
which the interval excludes zero.}
\begin{tabular}{lrrrr}
\toprule
stratum & $\Delta$ & 95\% CI & Ep. & Redraws \\
\midrule
nearest hazard near ($d < 0.62$) & $-0.342$ & $[-1.065, +0.334]$ & 38 & 10\% \\
nearest hazard far ($d > 0.93$)  & $\mathbf{-1.369}$ & $[-1.785, -0.958]$ & 44 & 100\% \\
\midrule
something within $\pm 45^\circ$ frontal cone & $-0.786$ & $[-1.238, -0.341]$ & 67 & 100\% \\
nothing within $\pm 45^\circ$ frontal cone   & $\mathbf{-1.797}$ & $[-2.402, -1.204]$ & 8 & 99\% \\
\bottomrule
\end{tabular}
\end{table}

\subsection{Exhaustive partition distribution}
\label{app:partfig}

The partition result drawn as a distribution, the form that shows how little the
split matters.
\begin{figure}[h]
\centering
\includegraphics[width=0.52\textwidth]{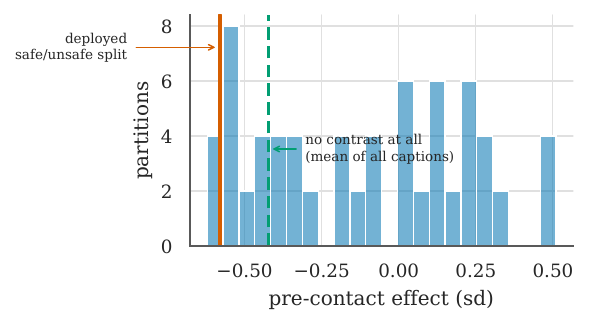}
\caption{All $\binom{8}{4}=70$ ways of splitting the caption set into a positive
and a negative group, scored identically. The deployed split (solid line) ranks
third of the seventy ($p = 0.043$), and the distribution is centred below zero
with a mean of $-0.113$, but it runs from $-0.620$ to $+0.510$, so a sizeable
minority of splits reverse sign. Removing the contrast entirely (dashed line)
retains $73\%$ of the deployed split's effect.}
\label{fig:partitions}
\end{figure}

\subsection{Closed-loop arms in full}
\label{app:placebotab}

The four trained arms in full. We detect no difference between constant and
frame-dependent $\kappa$. All three VLM-shaped arms have lower point estimates
than the VLM-free baseline on both safety measures.
\begin{table}[h]
\centering
\caption{No performance difference is detected when the VLM confidence is
replaced by a constant. One evaluation sweep, $200$ deterministic episodes per run, seed as
the unit. The two placebo arms carry $15$ seeds and the two reference arms $10$.
$\kappa$ scales only the reward bonus here, never the cost channel. The
evaluation scenarios alias onto training scenarios (defect D1 of
Appendix~\ref{app:defects}), identically across arms, so the between-arm
differences stand while the absolute rates are optimistic.}
\label{tab:placebo}
\begin{tabular}{llrrrrr}
\toprule
arm & $\kappa$ & seeds & catastrophe & violation & $J_C$ & $J_R$ \\
\midrule
baseline  & no VLM at all              & $10$ & $0.281$ & $0.353$ & $142.2$ & $-623.2$ \\
deployed  & CLIP, current frame        & $10$ & $0.192$ & $0.239$ & $102.6$ & $-460.4$ \\
constant  & fixed at $0.850$           & $15$ & $0.176$ & $0.217$ & $90.9$  & $-412.5$ \\
yoked     & real trace, wrong episode  & $15$ & $0.210$ & $0.263$ & $109.9$ & $-485.1$ \\
\bottomrule
\end{tabular}
\end{table}

\subsection{The language pathway in full}
\label{app:language}

Text is doing something a generic novelty score is not. Over the \emph{same} frozen
embeddings, $k$-NN distance, Mahalanobis distance, PCA residual and a
random-network-distillation error reach band deviations of $+0.227$, $+0.029$,
$-0.066$ and $+0.027$, against the margin's $-1.099$. The margin beats every one
under Benjamini--Hochberg (largest gap $+1.070$, smallest $+0.871$, all
$p \le 0.0008$), and beats each of five random directions drawn in the same embedding
space, whose largest $p$ is $0.016$. The tested language direction therefore carries information beyond these generic embedding scores.

They are not, however, doing what the construction assumes. Over all
$\binom{8}{4} = 70$ partitions of the caption set, the deployed grouping scores
$-0.580$ and ranks third (permutation $p = 0.043$). That is not arbitrary, but it
is far weaker than the construction implies, and two regroupings that mix safety
with danger descriptions do better. The distribution has mean $-0.113$ but includes both signs, so the effect is not
unique to the semantic split. Scoring the mean of all eight
captions, with no split, retains
$-0.422\,[-0.636,-0.213]$, or $\mathbf{73\%}$ of the deployed margin. The contrast
contributes something, the shared direction most of it.

Per caption, the effect is carried by those naming the track layout (``off the
track and unsafe'' $-0.532$; ``centered on the track'' $-0.426$; ``driving in the
wrong direction'' $-0.338$). Neither estimate is resolved for the two captions that name collision without
describing the track (``about to crash into the barrier'' $-0.258$; ``the car is
colliding with obstacles'' $-0.081$). Under ViT-B/32 both group means fall
before contact (positive $-0.292$, negative $-0.343$) and nothing rises.
Section~\ref{sec:geometry} shows that this last property is specific to this
encoder and bank, and is not general.

\paragraph{Comparisons with pixel statistics are unresolved.} Mean
pixel variance and mean edge density, computed from the same frames without
encoder or captions, reach $-0.532$ and $-0.668$, roughly half and two-thirds of the
margin's effect. The margin's advantage over them is unresolved:
$+0.566\,[-0.015,+1.204]$, $p = 0.060$ against pixel variance, and
$+0.430\,[-0.054,+0.889]$, $p = 0.080$ against edge density. At three policies these
are underpowered comparisons, not informative nulls, and we do not claim the
quantities are equivalent. These controls show that simple image statistics also change before contact,
but the present sample does not resolve their differences from the margin.

\paragraph{A note on two controls we initially got wrong.} An earlier version of
this analysis drew three regroupings at random, saw that they scored near zero,
and concluded that the semantic grouping carried the effect. Three draws from a
distribution with mean $-0.113$ spanning $-0.620$ to $+0.510$ cannot support that
conclusion. An exhaustive enumeration over three policies then suggested the
opposite, rank~8 of 70 at $p = 0.114$, which the full nine-policy enumeration
reported here also does not support. We give the sequence
because it is the honest one. The sampled control was underpowered in the direction
that flattered the hypothesis, the first exhaustive control in the direction that
flattered its refutation.

\subsection{Moderation by hazard geometry, plotted}

The lead-lag profile and the geometric stratification, drawn instead of tabulated.
The plateau shape and the reversed ordering are easier to see than to read.
\begin{figure}[h]
\centering
\begin{minipage}[h]{0.48\textwidth}\centering
  \includegraphics[width=0.70\linewidth]{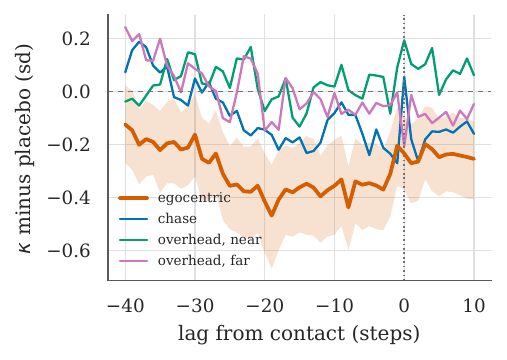}
\end{minipage}\hfill
\begin{minipage}[h]{0.48\textwidth}\centering
  \includegraphics[width=0.66\linewidth]{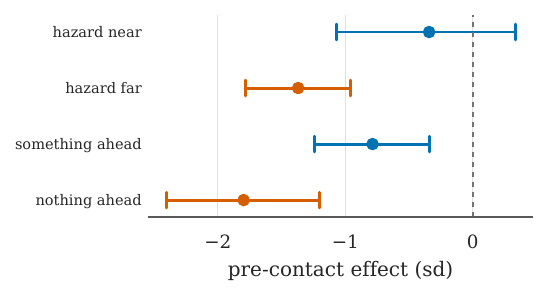}
\end{minipage}
\caption{\textbf{Left:} the measured signal on FormulaOne under CLIP ViT-B/32, per
viewpoint, against lag from contact, relative to matched placebo points, over ten
policies across three algorithms ($349$ isolated onsets in $179$ episodes). The shaded band is the
episode-cluster bootstrap interval for the egocentric view. Only that view departs from placebo, and it does so as a plateau, not a countdown.
\textbf{Right:} the same effect stratified by hazard geometry at the pre-contact
points, on the three passive-observer policies that carry hazard logs. The
deviation is largest in the \emph{far} stratum and in the one where
\emph{nothing} lies within a $\pm45^\circ$ frontal cone, and the near stratum's
interval spans zero. A hazard detector would show the reverse ordering. We do
not test the differences between strata.}
\label{fig:leadlag}
\end{figure}

\subsection{Viewpoint generalisation and the depiction test}
\label{app:generalisation}

The viewpoint result above is a fact about FormulaOne. Before running them, we recorded the prediction that the egocentric-only pattern would
repeat in two further egocentric environments. It does not.

\emph{CarButton1} shows the effect in three of its four views, the
overhead-far interval including zero, and no detected viewpoint difference: none of the
three paired contrasts against egocentric is significant ($-0.069$, $+0.006$,
$-0.097$). Safety-Gymnasium hazards are non-colliding, so an onset in these two
environments is entry into a hazard zone and not a physical contact.
\emph{PointGoal1} shows it only overhead, with
egocentric at $-0.001\,[-0.237,+0.245]$, a near-zero estimate where FormulaOne's effect
is strongest. The prediction is falsified.

\paragraph{A depiction account tested causally.} Across the three
environments the direction tracks how the hazard is rendered. FormulaOne's upright
barriers are visible from the driver's seat. \emph{PointGoal1}'s hazards are
flat two-centimetre floor discs, nearly invisible edge-on and obvious from above.
That suggests the signal exists wherever the frame \emph{depicts} the structure
that co-varies with failure, and that viewpoint matters only because it determines
depiction. We tested it causally instead of leaving it as a story.

Safety-Gymnasium hazards do not participate in collisions, so raising their half-height from $0.01$ to $0.30$
changes what the cameras see and nothing else. The manipulation check confirms
it: the two arms produce bit-identical trajectories while $\kappa$ moves
substantially (mean $|\Delta\kappa| = 0.1395$ across views). If depiction
decides which view carries the signal, raising the hazards must move the advantage
toward egocentric. With the analysis fixed in advance, over the $89$ episodes usable in both arms, the
egocentric-minus-overhead advantage is $+0.554\,[+0.253,+0.867]$ flat and
$+0.470\,[+0.175,+0.765]$ tall, a shift of $-0.084\,[-0.387,+0.228]$. The interval includes no shift, and the prespecified support criterion is not
met. The depiction account is therefore unsupported by its causal test.

\paragraph{What survives.} Informativeness is a functional of the observation map.
Across renderings of one identical simulator state it runs from substantial to nil,
and in \emph{PointGoal1} it reverses outright. That is the claim the viewpoint
experiment licenses, and these environments strengthen it. Two things do
\emph{not} survive: the privilege of the egocentric view, and any rule for
predicting in advance which rendering will carry the signal. It is egocentric only in
FormulaOne, present in three of four views in \emph{CarButton1}, overhead only
in \emph{PointGoal1}, and absent in MetaDrive. For a practitioner this is the more
demanding conclusion. In these experiments, camera labels alone do not predict
which view is informative. The score must be evaluated for each environment and
renderer.

These two experiments use a smoothed random policy instead of trained agents, one
policy distribution per arm, and their own caption banks. They establish the
failure of the generalisation, not a replacement account.

\subsection{Pre-contact deviation by camera}
\label{app:cameratab}

The viewpoint table in full, with $\kappa$'s coefficient of variation per view.
That column argues against a dispersion floor, since the chase camera is the most
variable of the four and still carries no effect.

\begin{table}[t]
\centering
\caption{Pre-contact deviation by camera. Ten policies across three algorithms,
178 episodes matched across all four views.}
\begin{tabular}{llrrr}
\toprule
view & description & $\kappa$ CV & $\Delta$ & 95\% CI \\
\midrule
\emph{vision}    & first-person, head-mounted & 0.224 & $\mathbf{-0.447}$ & $[-0.706, -0.204]$ \\
\emph{track}     & chase camera following agent & 0.729 & $-0.137$ & $[-0.338, +0.062]$ \\
\emph{fixednear} & static overhead, near      & 0.436 & $-0.080$ & $[-0.248, +0.088]$ \\
\emph{fixedfar}  & static overhead, far       & 0.661 & $-0.049$ & $[-0.262, +0.158]$ \\
\bottomrule
\end{tabular}
\end{table}

\subsection{Caption-bank geometry, caption rewrites, and placebo stability}
\label{app:geometry}

Table~\ref{tab:textgeom} measures the caption bank in each encoder's text space,
before any image is scored. \emph{across} is the mean cosine between a positive
and a negative caption, and the two \emph{within} columns are the mean cosines
inside each group. For ViT-B/32 the across-group value equals the within-positive
value to three decimals, indicating little separation between the two groups.
The residual column reports $\sqrt{1-\rho^2}$ for centroid cosine $\rho$.

\begin{table}[h]
\centering
\caption{Caption-bank geometry in each encoder's text space.}
\label{tab:textgeom}
\small
\begin{tabular}{lrrrrrr}
\toprule
encoder & within $+$ & within $-$ & across & centroids & residual & 2 labels \\
\midrule
CLIP ViT-B/32 & 0.883 & 0.898 & 0.883 & 0.962 & 0.274 & 0.846 \\
CLIP ViT-L/14 & 0.808 & 0.845 & 0.794 & 0.913 & 0.408 & 0.817 \\
\bottomrule
\end{tabular}
\end{table}

Table~\ref{tab:capcontrols} scores five caption sets on the same stored frames
with CLIP ViT-B/32, so trajectories, contacts, placebo matching and bootstrap are
held fixed and only the wording changes. The negated set asserts the opposite of
the deployed positive set and behaves as it does. The absurd set is nonsense about
the same subject and still falls. Only the off-topic set, which names a different
subject entirely, reverses.

\begin{table}[h]
\centering
\caption{Pre-contact deviation by caption set, and each set's mean within-episode
correlation with the deployed positive group.}
\label{tab:capcontrols}
\small
\begin{tabular}{llrr}
\toprule
caption set & example & band effect [95\% CI] & corr. \\
\midrule
deployed safe   & ``centered on the track and driving safely'' & $-0.527\,[-0.865,-0.191]$ & 1.000 \\
deployed danger & ``about to crash into the barrier''          & $-0.233\,[-0.546,+0.099]$ & 0.897 \\
negated safe    & ``not centered \ldots\ and not driving safely'' & $-0.507\,[-0.844,-0.170]$ & 0.951 \\
absurd          & ``the racecar is made of cheese''            & $-0.363\,[-0.732,+0.004]$ & 0.742 \\
off-topic       & ``a bowl of fruit on a wooden table''        & $+0.441\,[+0.031,+0.857]$ & $-0.262$ \\
\bottomrule
\end{tabular}
\end{table}

Each episode's placebo reference points are one random draw, so
Table~\ref{tab:redraw} repeats every arm over $200$ independent draws. It reports
the median estimate, the range across draws, and the share whose interval
excludes zero. Robustness to the placebo draw differs across the rows of Table~\ref{tab:geometry}. The danger-group rise under ViT-L/14 survives every draw, while the
generative model's survives roughly seven draws in ten and its two-label arm about
half. The margin's fall is stable in every arm. The redraws use deterministic episode-specific seeds.

\begin{table}[h]
\centering
\caption{Stability over $200$ placebo redraws: median [min, max] estimate and the
share of draws whose $95\%$ interval excludes zero.}
\label{tab:redraw}
\footnotesize
\setlength{\tabcolsep}{4pt}
\begin{tabular}{llrr}
\toprule
scorer & captions & danger group & margin \\
\midrule
CLIP ViT-B/32 & 8-caption bank & $-0.249\,[-0.427,-0.046]$ \ 27.5\% & $-0.909\,[-1.236,-0.684]$ \ 100\% \\
CLIP ViT-L/14 & 8-caption bank & $+0.785\,[+0.486,+1.017]$ \ 100\% & $-0.732\,[-0.920,-0.436]$ \ 100\% \\
Qwen2-VL-7B   & 2 group labels & $+0.270\,[+0.026,+0.488]$ \ 51\% & $-0.416\,[-0.653,-0.076]$ \ 86.5\% \\
Qwen2-VL-7B   & 8-caption bank & $+0.355\,[+0.124,+0.585]$ \ 72\% & $-0.628\,[-0.933,-0.322]$ \ 99\% \\
\bottomrule
\end{tabular}
\end{table}

The prompt swap inside CLIP ViT-B/32 uses the per-frame image embeddings stored by
the embedding experiment, which reproduce the published egocentric scores to
$2.7\times10^{-7}$. The check is asserted before any swapped-prompt number is
computed, and the analysis refuses to report if it fails. Scoring follows the deployed rule: $r$ and $c$ are affine group-mean cosines, and
the margin is $2(r-c)$.

\subsection{The frames, the prompts, and what each scorer reported on them}
\label{app:frames}

One approach to contact, shown as the scorers received it. A rule that cannot see
any score picks the episode: the first isolated onset in the first matched policy
with at least $30$ contact-free steps of history. The frames come from a replay. Its per-step
environment reward and cost were checked against the stored trajectory before any
image was written, so the pictures belong to the numbers beside them.

\begin{figure}[p]
\centering
\includegraphics[width=\textwidth]{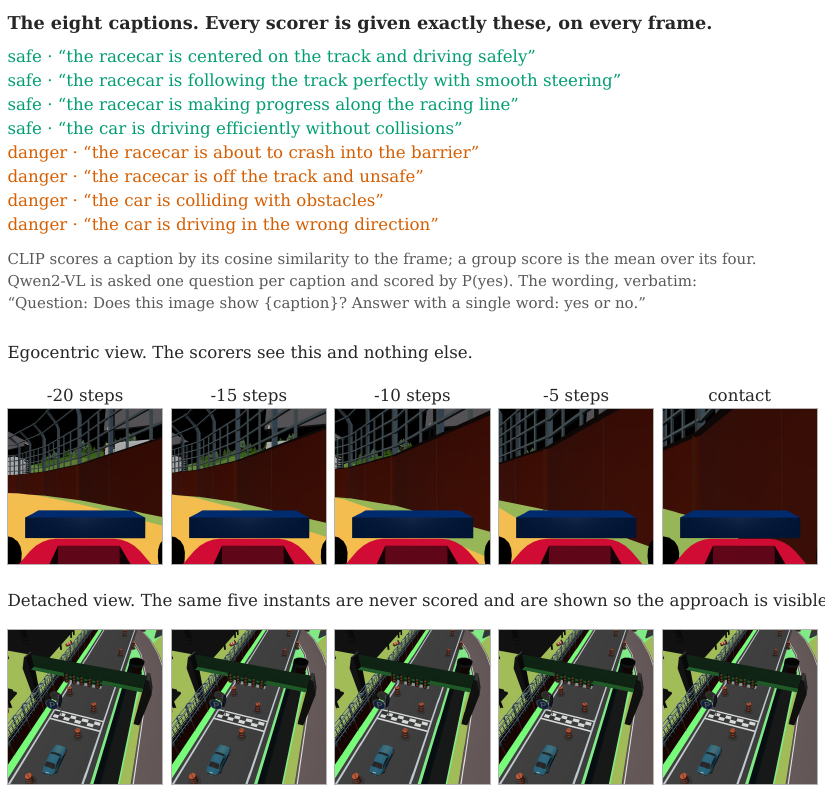}
\caption{The captions, and the frames they were scored against. \textbf{Top:} the
eight captions verbatim, and how each scorer consumes them. \textbf{Middle:} the
egocentric frames the scores are computed from, at four lags before contact and at contact.
\textbf{Bottom:} a detached overhead view of the same five instants, which no score
sees. We include it so a reader can tell what is happening in the world. The barrier fills the
view as the track recedes. The frame is leaving the scene the captions describe,
which is what Section~\ref{sec:theory} says the score reads. These images come from a $1024 \times 1024$ render of the replayed
trajectory, so that they stay legible when enlarged. The scorers were given the
$256 \times 256$ render of the identical simulator state, which CLIP resizes to
$224 \times 224$ (Appendix~\ref{app:encoderinput}).}
\label{fig:frames}
\end{figure}

\begin{figure}[p]
\centering
\includegraphics[width=\textwidth]{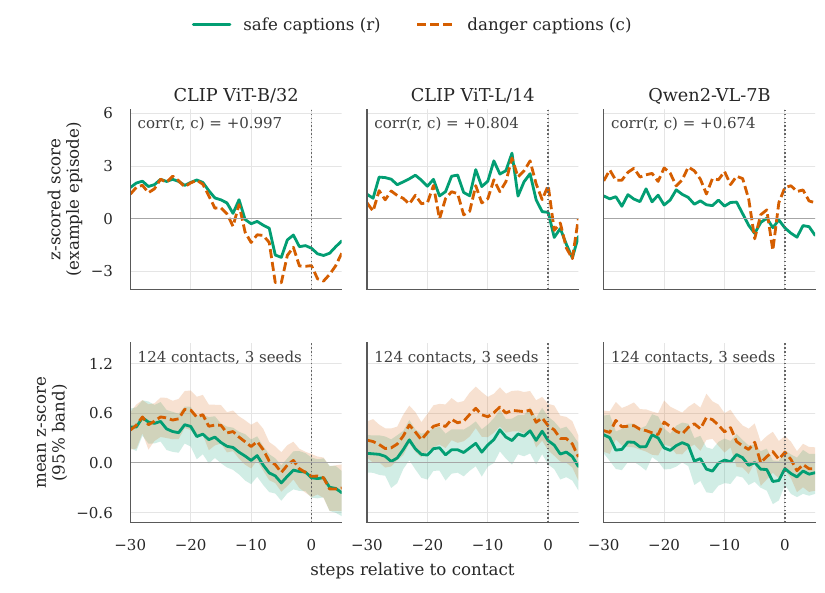}
\caption{What each scorer reported on those frames. \textbf{Top row:} the safe and
danger group scores over the approach shown in Figure~\ref{fig:frames}, $z$-scored
within episode. \textbf{Bottom row:} the same quantities averaged over the
$124$ of $130$ isolated contacts in the three matched policies that have a
complete $-30$ to $+5$ window. Its $95\%$ bands are bootstrapped over onsets and
are illustrative. Inferential results in the tables use episode-cluster
bootstraps. All three scorers read
the same eight captions on the same frames. Over exactly these frames ViT-B/32's two
group scores are nearly one series (correlation $+0.997$ across the plotted
window) and both fall, while ViT-L/14's danger group separates from its safe group
($+0.804$). The two traces in each panel are $z$-scored separately, so a near-unit
correlation still leaves visible differences in amplitude.}
\label{fig:traces}
\end{figure}

\section{Environment and evaluation defects}
\label{app:defects}

Four defects affected the evaluation. \textbf{D1. Scenario aliasing.} MetaDrive
interprets reset values as scenario indices, and reducing out-of-range indices
modulo $10^4$ maps the intended held-out set back onto training scenarios. For one MetaDrive-Hard policy with traffic fixed, the rollout obtained after
clamping index $10007$ to $7$ is bit-identical to a second rollout at index
$7$. \textbf{D2. Traffic
re-randomisation.} For one MetaDrive-Hard policy and scenario, two replays differ at all $1000$
steps. Disabling traffic re-randomisation makes them bit-identical, so paired
evaluation under this configuration requires fixed traffic. \textbf{D3. Hardware-dependent replay.} Five replays on the hardware group
that generated the references reproduce exactly, while none of seven replays
on other tested groups do. \textbf{D4. Augmented return.} Estimated from the last five training epochs, the VLM bonus contributes
$99.1\%$--$99.8\%$ of logged return in nine FormulaOne-L2 runs without
confidence gating, while recomputed environment returns overlap the
no-VLM range. Reward-shaped methods must therefore report environment-only
return separately. FormulaOne replay outputs record seeds, episode counts, hardware group, and
reference trajectories, and analyses omit outputs without a passing replay
check. Closed-loop validation separately checks run and episode counts.

\section{Limitations}
\label{app:limitations}

\paragraph{The egocentric result is specific to one environment.} The egocentric-only pattern appears in FormulaOne but not in the two additional
environments (Section~\ref{sec:generalisation}). Across the tested environments,
informativeness varies with the observation map, and no view is consistently
privileged. We have no account that predicts which view will carry it. The tested
hazard-visibility explanation was not supported under its prespecified
criterion.

\paragraph{Three scorers, one of them generative.} Section~\ref{sec:geometry}
compares two contrastive encoders and one generative VLM (Qwen2-VL-7B), too few to
characterise generative scorers. The Qwen two-label danger-score increase is
resolved in only $51\%$ of placebo redraws. We compare scorer class with the
prompt set fixed and prompt sets within ViT-B/32, but do not evaluate ViT-L/14
with the two labels. The larger encoder's
resolved sign reversal in the near-overhead view remains unexplained.

\paragraph{The band results are egocentric, and rest on one placebo draw per
episode.} Every number in Table~\ref{tab:geometry} is egocentric. Under ViT-B/32 no event-locked deviation is resolved in the three detached
views. Caption rewrites and partitions were tested only in the egocentric view
with ViT-B/32. ViT-L/14 was scored in every view and shows a near-overhead
reversal in $\kappa$ whose interval excludes zero, while the corresponding
margin interval includes zero (Section~\ref{sec:viewpoint}). The reference
points are drawn once per episode, so the table reports the share of $200$
redraws in which each interval excludes zero instead of one draw's verdict, and
the robustness of the danger-score increase varies across configurations. We
did not redraw the ViT-B/32 two-label row. The strata of
Section~\ref{sec:moderation} carry the same caveat, and the near-hazard estimate
is sensitive to the placebo draw.

\paragraph{Camera manipulation is within one renderer.} The four-camera tests use three MuJoCo tasks, so rendering style is fixed within
each task. MetaDrive contributes one view from another engine. Across
environments, viewpoint remains confounded with rendering style and hazard
appearance.

\paragraph{Open-loop measurement, with one closed-loop exception.} We measure on
trajectories the signal does not influence, which makes them interpretable but
says nothing directly about a policy trained on it.
Section~\ref{sec:closedloop} answers that separately, by training, and its scope
is narrow: one environment, one algorithm, one placebo family. The first-person
deviation of $-0.447$ standard deviations at ten policies holds within two of the
three algorithms and is modest, and we do not claim it suffices as a safety
signal. The third, CPO, contributes two policies, and its estimate is positive
with an interval spanning zero.

\paragraph{Scope relative to prior anticipatory constructions.} A recent
construction sends an epoch-averaged safety term through a scalar Lagrange
multiplier~\citep{tetteh2026signals}. Proposition~\ref{prop:budget} shows that
such a term only retunes the cost budget if its occupancy-level mean is affine
in the true cost for every feasible occupancy measure. We do not test that
condition for the prior construction. The proposition therefore identifies a
conditional mechanism limitation and does not challenge its reported safety
improvements.

\end{document}